\documentclass{article}
\usepackage{spconf,amsmath,graphicx}
\usepackage{algorithm}
\usepackage{algpseudocode}
\usepackage{array}
\usepackage[caption=false,font=normalsize,labelfont=sf,textfont=sf]{subfig}
\usepackage{textcomp}
\usepackage{stfloats}
\usepackage{url}
\usepackage{verbatim}
\usepackage{cite}
\usepackage{nccmath}
\usepackage{amssymb}
\usepackage{mathtools}
\usepackage{amsthm}
\usepackage{multirow} 
\usepackage{color}
\usepackage{booktabs}
\usepackage{cleveref}

\newtheorem{theorem}{Theorem}

\newtheorem{corollary}{Corollary}

\newcommand{\gen}{\operatorname{gen}}
\title{Federated PAC-Bayesian Learning on Non-IID data}
\name{Zihao Zhao$^{\star}$ \qquad Yang Liu$^{\dagger \ddagger}$ \qquad Wenbo Ding$^{\star  \ddagger}$ \qquad Xiao-Ping Zhang$^{\star}$ \thanks{This work was supported by the National Key R\&D Program of China under Grant No.2022ZD0160504, by Tsinghua Shenzhen International Graduate School-Shenzhen Pengrui Young Faculty Program of Shenzhen Pengrui Foundation (No. SZPR2023005), and by Tsinghua-Toyota Joint Research Institute inter-disciplinary Program and Tsinghua University (AIR)-Asiainfo Technologies (China) Inc. Joint Research Center under grant No. 20203910074.}}
\address{$^{\star}$Tsinghua-Berkeley Shenzhen Institute, $^{\dagger}$Institute for AI Industry Research, $^{\ddagger}$Shanghai AI Lab}
\begin{document}
%
\maketitle



\begin{abstract}
    Existing research has either adapted the Probably Approximately Correct (PAC) Bayesian framework for federated learning (FL) or used information-theoretic PAC-Bayesian bounds while introducing their theorems, but few considering the non-IID challenges in FL. Our work presents the first non-vacuous federated PAC-Bayesian bound tailored for non-IID local data. This bound assumes unique prior knowledge for each client and variable aggregation weights. We also introduce an objective function and an innovative Gibbs-based algorithm for the optimization of the derived bound. The results are validated on real-world datasets.
\end{abstract}


\section{Introduction}
To address privacy concerns in distributed learning, \textit{federated learning} (FL) has emerged as a viable solution, enabling multiple local clients to collaboratively train a model while retaining their private data and without sharing it~\cite{mcmahan2017communication, zhao2023towards}. However, in real-world scenarios, data across different devices is not identically and independently distributed (non-IID), which poses challenges in model training and convergence~\cite{zhao2018federated}.

Significant efforts have been made to improve performance and analyze convergence in non-IID FL~\cite{li2019convergence}, but few have provided theoretical guarantees by establishing generalization bounds. Most existing FL generalization analyses rely on the Probably Approximately Correct (PAC) Bayesian theory, first formulated by McAllester~\cite{mcallester1998some, mcallester1999pac}. Building on the McAllester's bound, these analyses typically compute local bounds or apply existing PAC-Bayesian bounds directly, overlooking the non-IID nature of FL. This approach is flawed, as the PAC-Bayesian framework assumes that each data point is IID, ignoring non-IID data and directly employing the PAC-Bayesian theory, which potentially results in inaccurate or overly relaxed bounds. Consequently, techniques developed for the PAC-Bayesian framework are not directly applicable to non-IID FL. Therefore, this work aims to advance the theoretical underpinnings of non-IID FL.

\textbf{Related works.}
The PAC-Bayesian framework has been extensively researched in recent years~\cite{seeger2002pac, catoni2007pac, oneto2020randomized}, yielding tighter and non-vacuous bounds. However, there has been limited exploration in the context of FL. Some studies have proposed information theoretic-based PAC-Bayesian bounds using Rate-distortion theory to prove generalization bounds~\cite{sefidgaran2022rate, barnes2022improved}, providing an information-theoretic perspective on enhancing generalization capacity. Others have followed McAllester’s approach, attempting to directly apply the FL paradigm to the bound. For example, the authors in~\cite{sefidgaran2023federated, chor2023more} applied McAllester’s bound in a multi-step FL scenario; Omni-Fedge~\cite{kesanapalli2021federated} used the PAC-Bayesian learning framework to construct a weighted sum objective function with a penalty, considering only a local client bound instead of the entire system, which precludes obtaining global information; and FedPAC~\cite{zhang2023federated} employed PAC learning to balance utility, privacy, and efficiency in FL. However, these approaches do not account for the non-IIDness of FL.

\textbf{Our contributions.} \textbf{First}, we derive a federated PAC-Bayesian learning bound for non-IID local data, providing a unified perspective on federated learning paradigms. To the best of our knowledge, this is the first non-vacuous bound for a model averaging FL framework. Specifically, due to the non-IID nature of clients, we assume that each client has unique prior knowledge rather than a common one. \textbf{Additionally}, the aggregation weights for non-IID clients vary instead of being uniform. Based on the derived bound, we define an objective function that can be computed by each local client rather than on the server and propose a Gibbs-based algorithm dubbed \textit{FedPB} for its optimization. This algorithm not only preserves the privacy of each client but also enhances efficiency. \textbf{Finally}, we validate our proposed bounds and algorithm on two real-world datasets, demonstrating the effectiveness of our bounds and algorithm.

\section{Problem Setting}
In this section, we introduce the federated PAC-Bayesian learning setting. The whole system comprises $K$ clients, each equipped with its own dataset $S_i = (x_i, y_i)_{i = 1}^{n} \subseteq (\mathcal{X}, \mathcal{Y})^{n}$ consisting of $n$ IID data points. Here $\mathcal{X}$ denotes the input space and $\mathcal{Y}$ denotes the output space. Each dataset $S_i$ is presumed to be drawn from an unknown data generating distribution $D_k^{\bigotimes n}$. Moreover, let $\ell: \mathcal{Z} \times \mathcal{W} \rightarrow \mathbb{R}^+$ be a given loss function and $h_k \in \mathcal{H}$ is a stochastic estimator on client $k$ where $\mathcal{H}$ is the hypothesis class. In the PAC-Bayesian framework, each client holds a tailored prior distribution $P_k$. The objective of each client is to furnish a posterior distribution $Q_k \in \mathcal{M}$, where $\mathcal{M}$ denotes the set of distributions over $\mathcal{H}$.
We then define the \textit{population risk}: 
\begin{equation}
    L(Q_1, \dots, Q_K) \triangleq \frac{1}{K} \sum_{k=1}^K \underset{h_k \sim Q_k}{\mathbb{E}} \underset{(x_{k}, y_{k}) \sim D_k}{\mathbb{E}}[\ell(h_k(x_k), y_k)],
\end{equation}
and the \textit{empirical risk}: 
\begin{equation}
    \hat{L}\left(Q_1, \dots, Q_K\right) \triangleq \frac{1}{nK} \sum_{k=1}^K  \underset{h_k \sim Q_k}{\mathbb{E}} \sum_{i=1}^n \ell\left(h_k\left(x_{k, i}\right), y_{k, i}\right),
\end{equation}
by averaging over the posterior distribution of each client.
In federated learning, each client will upload their posterior distributions to a central server, and then the server will aggregate the transmitted model in a weighted manner:
$$
\begin{aligned}
    \bar{P} =  \prod_{k=1}^K P_k^{p(k)}, \quad \bar{Q} =  \prod_{k=1}^K Q_k^{p(k)},
\end{aligned} 
$$
where $\bar{P}$ and $\bar{Q}$ are the global prior and posterior, respectively, and the averaging weight $p = (p(1), \dots, p(K))$ be a probability distribution on $\{1, \dots, K\}$. For the sake of generality, we can assume that $p(k) \in (0, 1)$ and $\sum_{k=1}^K p(k) = 1$. For intuition of this aggregation, we can see that minimizing the weighted objective function is actually equivalent to maximizing the logarithm of the corresponding posterior:
$
\min _h L(h) =\min _h \sum_{k=1}^K p(k) L_k(h)
=\max _h \ln \prod_{n=1}^N \\ p\left(h \mid \mathcal{D}_k\right)^{p(k)} .
$
In addition, we denote the Kullback-Leibler (KL) divergence as
$
D_{KL} (Q \| P) \triangleq \underset{Q}{\mathbb{E}} \left[ \log \frac{dQ}{dP} \right]
$ if $Q \ll P$ and $D_{KL} (Q \| P) = + \infty$ otherwise.

\section{Main theorem}
\label{sec3}
In this section, we will present our novel bounds on the non-IID FL scenario.

\begin{theorem}[Federated PAC-Bayesian learning bound]{
For any $\delta \in (0, 1]$, assume the loss function $\ell(\cdot, \cdot)$ is bounded in $[0, C]$, the following inequality holds uniformly for all posterior distributions $Q$ and for any $\delta \in (0, 1)$,
\begin{multline}
\thinmuskip=2mu
\medmuskip=4mu
\thickmuskip=2mu
\underset{S_1, \dots, S_K}{\mathbb{P}} \bigg\{ \forall {Q_1, \dots, Q_K}, L(Q_1, \dots, Q_K) \leq \hat{L} (Q_1, \dots, Q_K)   \\
\thinmuskip=0mu
\medmuskip=2mu
\thickmuskip=0mu
+ \frac{\sum_{k = 1}^K p(k) D_{KL}(Q_k \| P_k) + \log \frac{1}{\delta }}{\lambda}  + \frac{\lambda C^2}{8 K n} \bigg\} > 1 - \delta 
\end{multline}}
\label{thm1}
\end{theorem}
\begin{proof}
Let the local generalization error: $\gen (D_k, h_k) = \underset{(x_{k}, y_{k}) \sim D_k}{\mathbb{E}}[\ell(h_k(x_k), y_k)] - \frac{1}{n}\sum_{i=1}^n \ell\left(h_k\left(x_{k, i}\right), y_{k, i}\right)$, then the global generalization error: $\overline{\gen} (D, h) = L(Q_1, \dots, Q_K) \\-\hat{L}(Q_1, \dots, Q_K) = \frac{1}{K} \sum_{k=1}^K \underset{h_k \sim Q_k}{\mathbb{E}} \gen (D_k, h_k) $.
For any $\lambda > 0$ and $t > 0$, applying the Hoeffding's lemma to $\mathbb{E}[\ell_i] - \ell_i$, we have that, for each client $k$,
$$
\underset{S_k}{\mathbb{E}} \underset{P_k}{\mathbb{E}} \left[\mathrm{e}^{\frac{\lambda}{K} \gen(D_k, h_k)}\right] \leq \mathrm{e}^{\frac{\lambda^2 C^2}{8 K^2 n}}.
$$
Since each $S_k$ may come from different $D_k$. i.e., non-IID, we cannot directly plug this result to the PAC-Bayesian bound. Note that for each client $k \in [K]$, $P_i$ is independent of $S_1, \dots, S_K$, we have that 
$$
\underset{S_1}{{\mathbb{E}}} \underset{P_1}{{\mathbb{E}}} \dots \underset{S_K}{{\mathbb{E}}} \underset{P_K}{{\mathbb{E}}} \left[\mathrm{e}^{\frac{\lambda}{K} \sum_{k=1}^K \gen(D_k, h_k) }\right] \leq \mathrm{e}^{\frac{\lambda^2 C^2}{8 K n}}.
$$
And we apply Donsker and Varadhan's variational formula~\cite{donsker1975variational} for $P_1, \dots, P_K$ to get:
\begin{multline}
    \underset{S_1, \dots, S_K}{{\mathbb{E}}}\left[e^{\underset{Q_1, \dots, Q_K}{\sup} \lambda \underset{Q_1}{{\mathbb{E}}} \dots \underset{Q_K}{{\mathbb{E}}}[\frac{1}{K}\sum_{k=1}^K \gen(D_k, h_k)]} \right.\\
    \left./ e^{D_{KL} \left( \prod_{k=1}^K Q_k^{p(k)} \|  \prod_{k=1}^K P_k^{p(k)} \right)}\right] \leq \mathrm{e}^{\frac{\lambda^2 C^2}{8 K n}}.
    \label{eq4}
\end{multline}
Recall the definition of the global generalization error:
$$
\overline{\gen}(D, h) =  \frac{1}{K} \sum_{k=1}^K \underset{h_k \sim Q_K}{{\mathbb{E}}} \gen (D_k, h_k) ,
$$
and note that $D_{KL} \left( \prod_{k=1}^K Q_k^{p(k)} \|  \prod_{k=1}^K P_k^{p(k)} \right) = 
\\ \sum_{k = 1}^K p(k) D_{KL}(Q_k \| P_k)$. Applying the Chernoff bound:
$$
\scriptsize
\begin{aligned}
&\underset{S_1, \dots, S_K}{\mathbb{P}} \left[ \underset{Q_1, \dots, Q_K}{\sup}  \lambda \overline{\gen}(D, h) -\sum_{k = 1}^K p(k) D_{KL}(Q_k \| P_k) -\frac{\lambda^2 C^2}{8 K n} > s\right]  \\
& \leq \underset{S_1, \dots, S_K}{\mathbb{E}} \left[\mathrm{e}^{\underset{Q_1, \dots, Q_K}{\sup}  \lambda \overline{\gen}(D, h)  -\sum\limits_{k = 1}^K p(k) D_{KL}(Q_k \| P_k)-\frac{\lambda^2 C^2}{8 K n} }\right] e^{-s} \\
& \leq \mathrm{e}^{-s} .
\end{aligned}
$$
Let $\delta = e^{-s}$, that is, $s = - \log \delta$. Thus, plug this into the above result we have that
\begin{multline*}
\underset{S_1, \dots, S_K}{\mathbb{P}} \bigg\{ \exists {Q_1, \dots, Q_K}, \overline{\gen}(D, h) >  \\
\shoveright{ \frac{1}{\lambda} \sum_{k = 1}^K p(k) D_{KL}(Q_k \| P_k) + \frac{\lambda C^2}{8 K n} + \frac{1}{\lambda} \log \frac{1}{\delta }\bigg\} \leq \delta}.
\end{multline*}
Therefore, we prove the statement by leveraging the complement of the probability.
\end{proof}

The RHS of \Cref{thm1} comprises two components: the \textbf{empirical term} and the \textbf{complexity term}. Note that our bound eschews the typical smoothness and convexity assumptions on the loss often made by other FL frameworks. Moreover, an intuition can be presented for \Cref{thm1} that the bound will be tighter with the increasing of clients scales, which is further corroborated by the evaluation in \Cref{53}.

\begin{corollary}[The choice of $\lambda$]
Suppose $\lambda \in \Xi \triangleq \{0, \dots, \xi\}$ and $|\cdot|$ denotes the cardinality of a set. For any $\delta \in (0, 1)$ and a properly chosen $\lambda$, with probability at least $1 - \delta$,
    \begin{multline}
        L(Q_1, \dots, Q_K) \leq \hat{L}(Q_1, \dots, Q_K) \\
        + C \sqrt{\frac{\sum_{k = 1}^K p(k) D_{KL}(Q_k \| P_k) + \log \frac{|\Xi|}{\delta}}{2 K n}}.
    \end{multline}
\label{coro}
\end{corollary}
\begin{proof}
    Suppose $\mathcal{S} = S_1 \cap S_2 \cap \dots \cap S_K$ and $\mathcal{Q} = Q_1 \cap Q_2 \cap \dots \cap Q_K$. Since we have the previous result \eqref{eq4} for some fixed $\lambda$, we can sum \eqref{eq4} over all $\lambda \in \Xi$:
    \begin{multline*}
        \underset{\lambda \in \Xi}{\sum} \underset{\mathcal{S}}{\mathbb{E}} \left[\mathrm{e}^{\underset{\mathcal{Q}}{\sup}  \lambda \overline{\gen}(D, h)  -\sum\limits_{k = 1}^K p(k) D_{KL}(Q_k \| P_k)-\frac{\lambda^2 C^2}{8 K n}}\right] \le |\Xi|,
    \end{multline*}
    which is equivalent to:
    \begin{multline*}
        \underset{\mathcal{S}}{\mathbb{E}} \left[\mathrm{e}^{\underset{\mathcal{Q}, \lambda \in \Xi}{\sup}  \lambda \overline{\gen}(D, h)  -\sum\limits_{k = 1}^K p(k) D_{KL}(Q_k \| P_k)-\frac{\lambda^2 C^2}{8 K n}}\right] \le |\Xi|.
    \end{multline*}
    Again, from the Chernoff bound:
    $$
    \thinmuskip=0mu
    \medmuskip=-0mu
    \thickmuskip=0mu
    \begin{aligned}
        &\underset{\mathcal{S}}{\mathbb{P}} \left[ \underset{\mathcal{Q}, \lambda \in \Xi}{\sup}  \lambda \overline{\gen}(D, h) -\sum_{k = 1}^K p(k) D_{KL}(Q_k \| P_k) -\frac{\lambda^2 C^2}{8 K n} > s\right]  \\
        & \leq \underset{\mathcal{S}}{\mathbb{E}} \left[\mathrm{e}^{\underset{\mathcal{Q}, \lambda \in \Xi}{\sup}  \lambda \overline{\gen}(D, h)  -\sum\limits_{k = 1}^K p(k) D_{KL}(Q_k \| P_k)-\frac{\lambda^2 C^2}{8 K n}}\right] e^{-s} \\
        & \leq |\Xi| \mathrm{e}^{-s} .
    \end{aligned}
    $$
    Solve $\delta = |\Xi| \mathrm{e}^{-s}$ to get:
    \begin{multline*}
    \underset{\mathcal{S}}{\mathbb{P}} \bigg\{ \exists {\mathcal{Q}, \lambda \in \Xi}, \overline{\gen}(D, h) > \\
    \frac{\sum_{k = 1}^K p(k) D_{KL}(Q_k \| P_k) + \log \frac{|\Xi|}{\delta}}{\lambda}
    + \frac{\lambda C^2}{8 K n} \bigg\} \leq \delta
    \end{multline*}
    Choosing a proper minimizer 
    $$
    \lambda^* = \sqrt{8 K n \left(\sum_{k = 1}^K p(k) D_{KL}(Q_k \| P_k) + \log \frac{|\Xi|}{\delta}\right)} / C,
    $$
    we obtain the bound of this corollary.
\end{proof}

\Cref{coro} generally offers a strategy for selecting the value of parameter $\lambda$ so that the complexity term can be minimized.

\section{FedPB: Optimize the upper bound}

We denote $\mathcal{L}_k = \underset{h_k \sim Q_k}{\mathbb{E}} \frac{1}{n} \sum_{i=1}^n \ell\left(h_k\left(x_{k, i}\right), y_{k, i}\right)$.
Consider the following local objective function:
\begin{equation}
    \mathcal{J} (Q_k) = \lambda \mathcal{L}_k + p(k) D_{KL}(Q_k \| P_k)
    \label{local_obj}
\end{equation}
In our methodology, we introduce FedPB for a general scenario. It comprises two phases, designed to iteratively optimize the priors and posteriors for each client. Notably, in contrast to previous studies, clients are not required to upload their private prior and posterior distributions to the server, ensuring their privacy.

\textbf{Phase 1 (Optimize the posterior).} Given a fixed parameter $\lambda > 0$ and the prior $P_k^{t}$ during the training epoch $t + 1$, we aim to optimize the posterior $Q_k^{t}$ as
$
\hat{Q}_{k}^{t+1} = \arg \min_{Q_k} \mathcal{J} (Q_k),
$
yielding the solution:
$$
\frac{d \hat{Q}_{k}^{t+1}}{d P_k^t}(h)=\frac{\exp \left(-\lambda \ell\left(h, z_i\right)\right)}{\mathbb{E}_{h \sim P_k^t}\left[\exp \left(-\lambda \ell\left(h, z_i\right)\right)\right]}.
$$

\textbf{Phase 2 (Optimize the prior).} Having derived the optimal posterior $\hat{Q}_{k}^{t+1}$, the prior can be updated by $\hat{P}_{k}^{t+1} = {Q}_{k}^{t}$ since it minimizes $D_{KL}(Q_k \| P_k)$.

\textbf{Link with personalized federated learning.} Since the prior $\hat{P}_{k}^{t+1} = {Q}_{k}^{t}$ and ${Q}_{k}^{t}$ is equal to the aggregated global posterior at epoch $t - 1 $, the prior can be viewed as the global knowledge. Optimizing the objective function \eqref{local_obj} minimizes the disparity between the global and local knowledge, which is a a prevalent personalization strategy in FL~\cite{zhang2022personalized}.

\textbf{Re-parameterization trick.} Utilizing the Bayesian neural network~\cite{springenberg2016bayesian} as the local model aligns with our setting where all parameters are random, and we are optimizing their posterior distribution. In particular, the prior $P_k$ and posterior $Q_k$ are defined as follows:
$$
\begin{aligned}
& P_k (w; \vartheta_{P_k}) = \mathcal{N} \left( w; \mu_{P_k}, \sigma_{P_k}^2 I_d \right) = \prod_{i=1}^d \mathcal{N}\left(w_i ; \mu_{P_k, i}, \sigma_{P_k, i}^2\right) \\
& Q_{k}(w; \vartheta_{Q_k}) = \mathcal{N} \left( w; \mu_{Q_{k}}, \sigma_{Q_{k}}^2 I_d \right) = \prod_{i=1}^d \mathcal{N}\left(w_i; \mu_{Q_{k}, i}, \sigma_{Q_{k}, i}^2\right)
\end{aligned},
$$with every model parameter $w_i$ being independent.  Computing the Gibbs posterior directly can be challenging, hence we select the gradient descent as an alternative way. The update rule of \eqref{local_obj} at round $t+1$ is:
$$
\thinmuskip=0mu
\medmuskip=-0mu
\thickmuskip=0mu
\begin{aligned}
    \mu_{Q_k, i}^{t+1} &= \mu_{Q_k, i}^{t} - \lambda \nabla_{\mu_{Q_k, i}} \mathcal{L}_k - \frac{p(k) (\mu_{Q_k, i} - \mu_{P_k, i})}{ \sigma_{Q_k, i}^2}, \\
    \sigma_{Q_k, i}^{t+1} &= \sigma_{Q_k, i}^{t} - \lambda \nabla_{\sigma_{Q_k, i}} \mathcal{L}_k + \frac{p(k) (\sigma_{P_k, i}^2 - \sigma_{Q_k, i}^2 + \left(\mu_{P_k, i}-\mu_{Q_k, i}\right)^2)}{\sigma_{Q_k, i}^3},
\end{aligned}
$$
where the parameter $\lambda$ can be regarded as the learning rate of the GD and the KL-divergence is the regularization term. Calculating the gradients $\nabla_{\mu_{Q_k, i}}$ and $\nabla_{\mu_{Q_k, i}}$ directly can be intricate, but the re-parameterization trick is capable of tackling this issue. Concretely, we translate the $h \sim Q_k$ to $\varepsilon \sim \mathcal{N}(0, I_d)$ and then compute the deterministic function $h = \mu + \sigma \odot \varepsilon$, where $\odot$ signifies an element-wise multiplication. As a result, we have $\nabla_{\mu_{Q_k, i}} \underset{h \sim Q}{\mathbb{E}} l(h) = \nabla_{\mu_{Q_k, i}} \underset{\varepsilon \sim q(\varepsilon)}{\mathbb{E}} l(\epsilon)$, indicating its computability using an end-to-end framework with back-propagation.



\section{Evaluation}
In this section, we demonstrate our algorithm and our theoretical arguments with the FL non-IID setting. Specifically, the aggregation weight $p(k)$ is defined as the sample ratio of client $k$ relative to the entire data size across all clients. For the global aggregation, the global mean and covariance are calculated by $\bar{\mu} = \sum_{k=1}^K p(k) \mu_k \sigma_k^{-2} / \sum_{k=1}^K p(k) \sigma_k^{-2}$ and $\bar{\sigma} = 1 / \sum_{k = 1}^K p(k) \sigma_k^{-2}$, respectively.
Furthermore, we utilize two real-world datasets: MedMNIST (medical image analysis)~\cite{medmnistv1} and CIFAR-10~\cite{krizhevsky2009learning}. For each dataset, we adopt three distinct data-generating approaches for local clients: 1) \textbf{Balanced}: each client holds an equal number of samples; 2) \textbf{Unbalanced}: varying sample counts per client (e.g., $[0.05, 0.05, 0.05, 0.05, 0.1, 0.1, 0.1, 0.1, 0.2, 0.2]$ for 10 clients); 3) \textbf{Dirichlet}: differing sample counts per client following a Dirichlet distribution~\cite{yurochkin2019bayesian}. Besides, the entire FL system encompasses $K=10$ clients, initializing their posterior models uniformly from a global posterior model.

Additionally, we deploy two versions of Bayesian neural networks: one with 2 convolutional layers for MedMNIST and another with 3 layers for CIFAR-10. The CrossEntropy loss serves as our loss function, and it is optimized by the Adam optimizer~\cite{kingma2014adam} with a learning rate of 1e-3.

\subsection{Bound evaluation}
To validate our bounds, we set the confidence bound on $1 - \delta = 95\%$. Our evaluation underscores a correlation between the generalization error and the complexity, emphasizing the tightness of our bound. Fig.~\ref{bound} illustrates an initial increase in the generalization error and a concurrent decrease in complexity during the early stages of the training process, attributed to empirical loss optimization. Subsequently, as neural network training advances, the KL-divergence stabilizes. Throughout this progression, we can observe that the generalization error is consistently bounded by the complexity value.

\begin{figure}[!htb]

\begin{minipage}[b]{.45\linewidth}
  \centering
  \centerline{\includegraphics[width=1\linewidth]{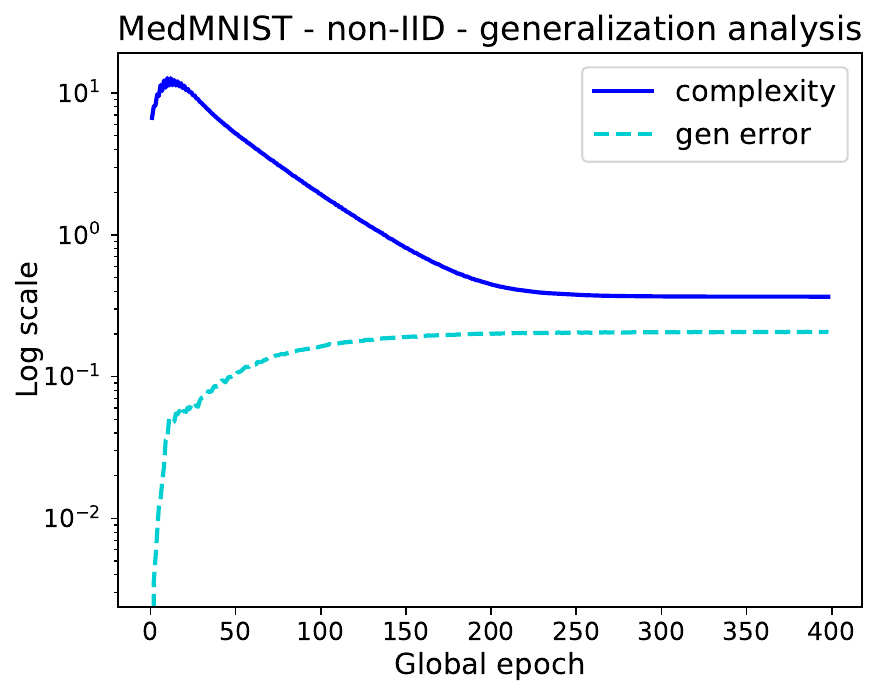}}
  \centerline{(a) Gen. for MedMNIST} \medskip
\end{minipage}
\hfill
\begin{minipage}[b]{0.45\linewidth}
  \centering
  \centerline{\includegraphics[width=1\linewidth]{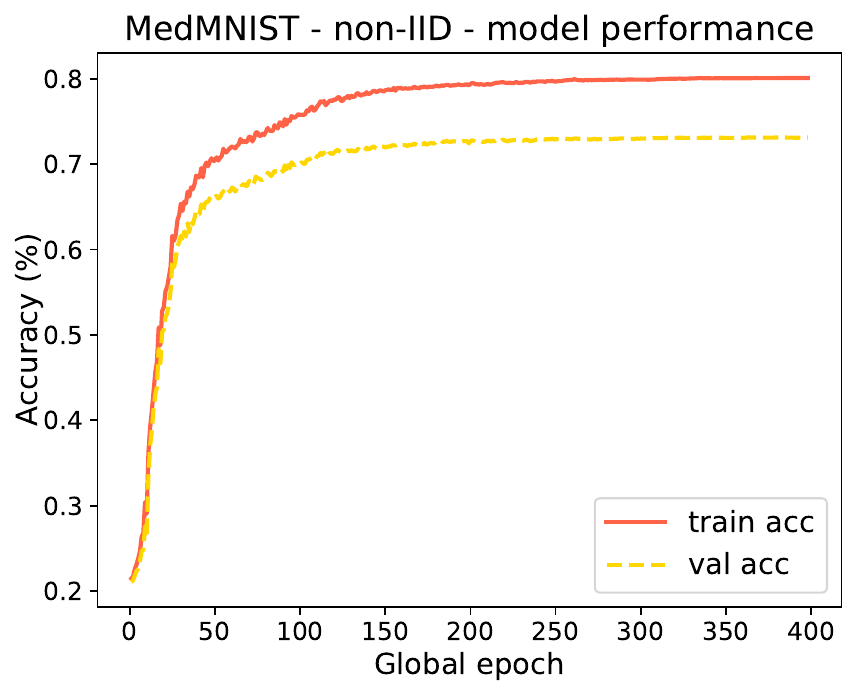}}
  \centerline{(b) Acc. for MedMNIST} \medskip
\end{minipage}

\begin{minipage}[b]{.45\linewidth}
  \centering
  \centerline{\includegraphics[width=1\linewidth]{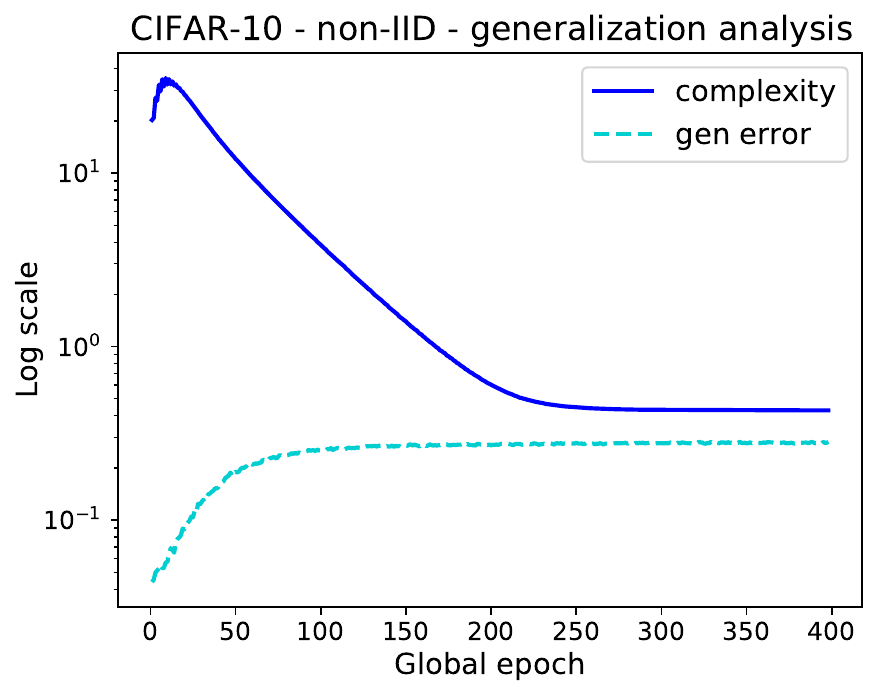}}
  \centerline{(c) Gen. for CIFAR-10} \medskip
\end{minipage}
\hfill
\begin{minipage}[b]{0.45\linewidth}
  \centering
  \centerline{\includegraphics[width=1\linewidth]{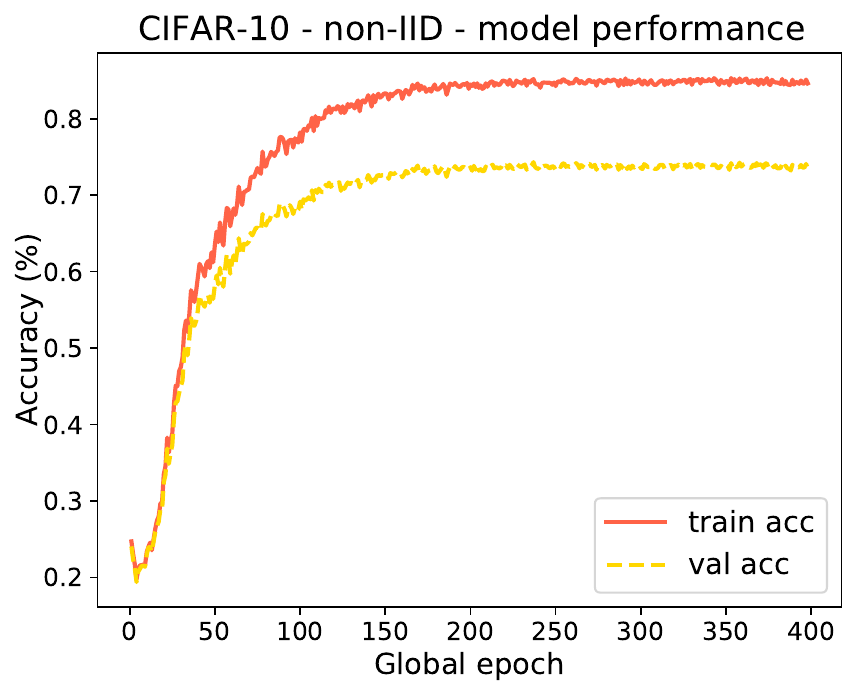}}
  \centerline{(d) Acc. for CIFAR-10} \medskip
\end{minipage}
\vspace{-18pt}
\caption{The results of the generalization error and model performance of FedPB over the Dirichlet generating method.}
\label{bound}
\end{figure}

\subsection{Data-dependent prior}
Here, we preform the ablation study of data-dependent (trainable) prior compared with data-independent (fixed, chosen before training) prior and report the mean ± standard deviation (std) accuracy of the global model in~\Cref{prior}, evaluated over multiple experimental seeds. The results demonstrate the superior efficacy of the data-dependent strategy over both datasets across all three scenarios. This superiority arises from the data-dependent prior's ability to harness more global knowledge, combined with its adaptability during training.
\vspace{-15pt}
\begin{table}[htb]
\centering
\caption{Model accuracy (\%) for the data-independent prior and data-dependent prior in three data-generating scenarios.}
\vspace{3pt}
\label{prior}
\resizebox{\columnwidth}{!}{%
\begin{tabular}{ccccccc}
\toprule
\multirow{2}{*}{Method} & \multicolumn{3}{c}{MedMNIST} & \multicolumn{3}{c}{CIFAR10} \\
\cmidrule(lr){2-4} \cmidrule(lr){5-7}
& Balanced & Unbalanced & Dirichlet & Balanced & Unbalanced & Dirichlet \\
\midrule
\begin{tabular}[c]{@{}c@{}}Data-\\ independent\end{tabular} & \(53.47 \pm 1.12\) & \(49.44 \pm 1.10\) & \(55.24 \pm 6.92\) & \(50.89 \pm 0.62\) & \(47.19 \pm 0.92\) & \(57.93 \pm 0.55\) \\
\begin{tabular}[c]{@{}c@{}}Data-\\ dependent\end{tabular} & \(\mathbf{77.10 \pm 4.25}\) & \(\mathbf{77.34 \pm 3.42}\) & \(\mathbf{77.48 \pm 4.75}\) & \(\mathbf{84.41 \pm 0.94}\) & \(\mathbf{79.39 \pm 0.56}\) & \(\mathbf{86.11 \pm 0.53}\) \\
\bottomrule
\end{tabular}
}
\end{table}
\vspace{-15pt}

\subsection{Different client scales}
\label{53}
Lastly, we assess the influence of varying client scales on our complexity bounds. As depicted in Fig.~\ref{diff_clients}, in the evaluation of both datasets with the Dirichlet generating method, by increasing the number of clients $K$ from 10 to 20 and 50, we observe a consistent decrease in the value of complexity. This observation aligns with  our analysis in~\Cref{thm1}.
\vspace{-10pt}
\begin{figure}[!htb]
    \centering
    \includegraphics[width=0.85\linewidth]{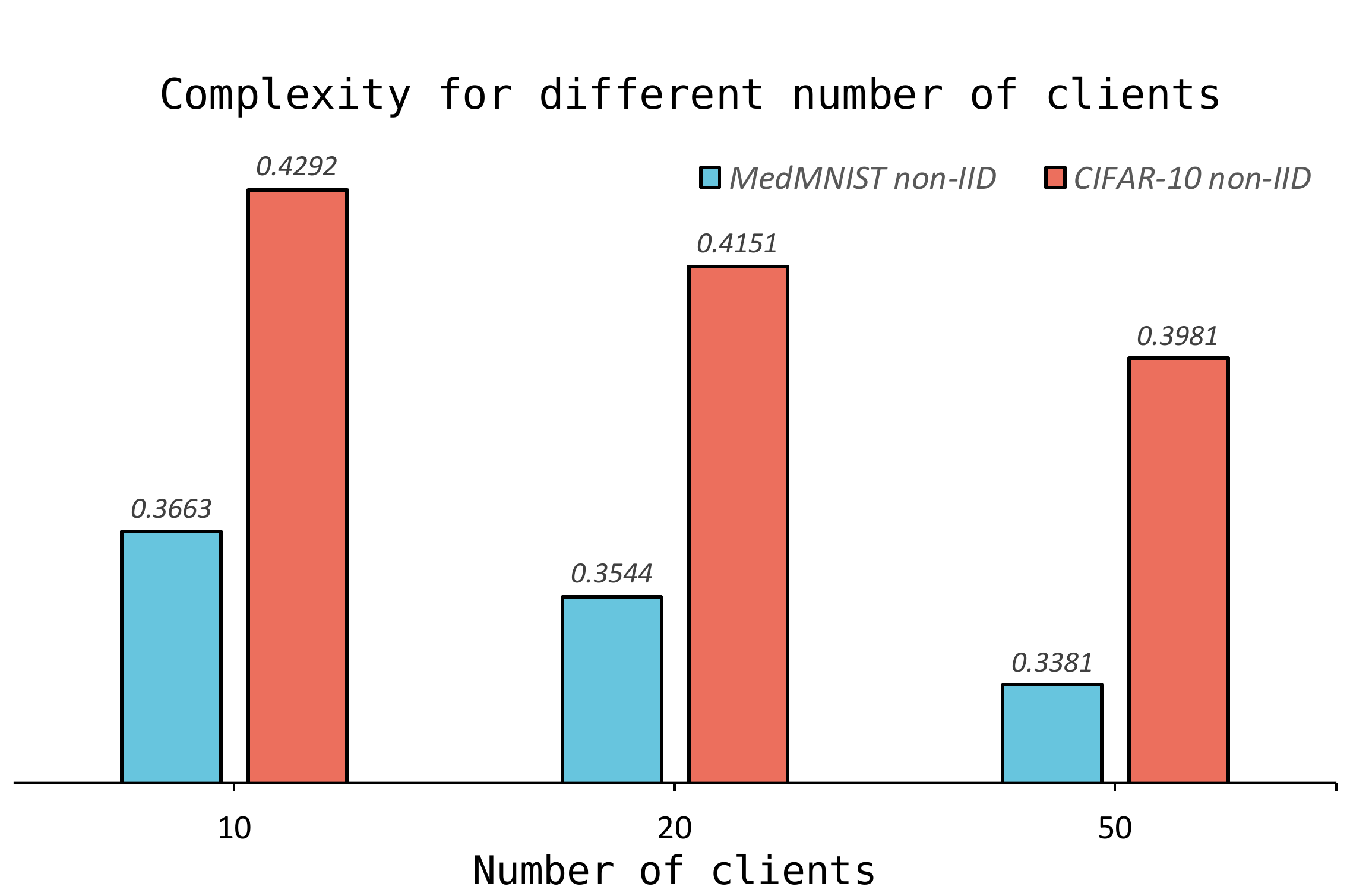}
    \vspace{-10pt}
    \caption{The impact of different client scales over FedPB on the value of the complexity term.}
    \label{diff_clients}
\end{figure}
\vfill\pagebreak

\bibliographystyle{IEEEbib}
\bibliography{refs}

\begin{thebibliography}{10}

\bibitem{mcmahan2017communication}
Brendan McMahan, Eider Moore, Daniel Ramage, Seth Hampson, and Blaise~Aguera
  y~Arcas,
\newblock ``Communication-efficient learning of deep networks from
  decentralized data,''
\newblock in {\em Artificial intelligence and statistics}. PMLR, 2017, pp.
  1273--1282.

\bibitem{zhao2023towards}
Zihao Zhao, Yuzhu Mao, Yang Liu, Linqi Song, Ye~Ouyang, Xinlei Chen, and Wenbo
  Ding,
\newblock ``Towards efficient communications in federated learning: A
  contemporary survey,''
\newblock {\em Journal of the Franklin Institute}, 2023.

\bibitem{zhao2018federated}
Yue Zhao, Meng Li, Liangzhen Lai, Naveen Suda, Damon Civin, and Vikas Chandra,
\newblock ``Federated learning with non-iid data,''
\newblock {\em arXiv preprint arXiv:1806.00582}, 2018.

\bibitem{li2019convergence}
Xiang Li, Kaixuan Huang, Wenhao Yang, Shusen Wang, and Zhihua Zhang,
\newblock ``On the convergence of fedavg on non-iid data,''
\newblock {\em arXiv preprint arXiv:1907.02189}, 2019.

\bibitem{mcallester1998some}
David~A McAllester,
\newblock ``Some pac-bayesian theorems,''
\newblock in {\em Proceedings of the eleventh annual conference on
  Computational learning theory}, 1998, pp. 230--234.

\bibitem{mcallester1999pac}
David~A McAllester,
\newblock ``Pac-bayesian model averaging,''
\newblock in {\em Proceedings of the twelfth annual conference on Computational
  learning theory}, 1999, pp. 164--170.

\bibitem{seeger2002pac}
Matthias Seeger,
\newblock ``Pac-bayesian generalisation error bounds for gaussian process
  classification,''
\newblock {\em Journal of machine learning research}, vol. 3, no. Oct, pp.
  233--269, 2002.

\bibitem{catoni2007pac}
Olivier Catoni,
\newblock ``Pac-bayesian supervised classification: the thermodynamics of
  statistical learning,''
\newblock {\em arXiv preprint arXiv:0712.0248}, 2007.

\bibitem{oneto2020randomized}
Luca Oneto, Michele Donini, Massimiliano Pontil, and John Shawe-Taylor,
\newblock ``Randomized learning and generalization of fair and private
  classifiers: From pac-bayes to stability and differential privacy,''
\newblock {\em Neurocomputing}, vol. 416, pp. 231--243, 2020.

\bibitem{sefidgaran2022rate}
Milad Sefidgaran, Romain Chor, and Abdellatif Zaidi,
\newblock ``Rate-distortion theoretic bounds on generalization error for
  distributed learning,''
\newblock {\em Advances in Neural Information Processing Systems}, vol. 35, pp.
  19687--19702, 2022.

\bibitem{barnes2022improved}
LP~Barnes, Alex Dytso, and H~Vincent Poor,
\newblock ``Improved information theoretic generalization bounds for
  distributed and federated learning,''
\newblock in {\em 2022 IEEE International Symposium on Information Theory
  (ISIT)}. IEEE, 2022, pp. 1465--1470.

\bibitem{sefidgaran2023federated}
Milad Sefidgaran, Romain Chor, Abdellatif Zaidi, and Yijun Wan,
\newblock ``Federated learning you may communicate less often!,''
\newblock {\em arXiv preprint arXiv:2306.05862}, 2023.

\bibitem{chor2023more}
Romain Chor, Milad Sefidgaran, and Abdellatif Zaidi,
\newblock ``More communication does not result in smaller generalization error
  in federated learning,''
\newblock {\em arXiv preprint arXiv:2304.12216}, 2023.

\bibitem{kesanapalli2021federated}
Sai~Anuroop Kesanapalli and BN~Bharath,
\newblock ``Federated algorithm with bayesian approach: Omni-fedge,''
\newblock in {\em ICASSP 2021-2021 IEEE International Conference on Acoustics,
  Speech and Signal Processing (ICASSP)}. IEEE, 2021, pp. 3075--3079.

\bibitem{zhang2023federated}
Xiaojin Zhang, Anbu Huang, Lixin Fan, Kai Chen, and Qiang Yang,
\newblock ``Probably approximately correct federated learning,''
\newblock {\em arXiv preprint arXiv:2304.04641}, 2023.

\bibitem{donsker1975variational}
Monroe~D Donsker and SR~Srinivasa Varadhan,
\newblock ``On a variational formula for the principal eigenvalue for operators
  with maximum principle,''
\newblock {\em Proceedings of the National Academy of Sciences}, vol. 72, no.
  3, pp. 780--783, 1975.

\bibitem{zhang2022personalized}
Xu~Zhang, Yinchuan Li, Wenpeng Li, Kaiyang Guo, and Yunfeng Shao,
\newblock ``Personalized federated learning via variational bayesian
  inference,''
\newblock in {\em International Conference on Machine Learning}. PMLR, 2022,
  pp. 26293--26310.

\bibitem{springenberg2016bayesian}
Jost~Tobias Springenberg, Aaron Klein, Stefan Falkner, and Frank Hutter,
\newblock ``Bayesian optimization with robust bayesian neural networks,''
\newblock {\em Advances in neural information processing systems}, vol. 29,
  2016.

\bibitem{medmnistv1}
Jiancheng Yang, Rui Shi, and Bingbing Ni,
\newblock ``Medmnist classification decathlon: A lightweight automl benchmark
  for medical image analysis,''
\newblock in {\em IEEE 18th International Symposium on Biomedical Imaging
  (ISBI)}, 2021, pp. 191--195.

\bibitem{krizhevsky2009learning}
Alex Krizhevsky, Geoffrey Hinton, et~al.,
\newblock ``Learning multiple layers of features from tiny images,''
\newblock 2009.

\bibitem{yurochkin2019bayesian}
Mikhail Yurochkin, Mayank Agarwal, Soumya Ghosh, Kristjan Greenewald, Nghia
  Hoang, and Yasaman Khazaeni,
\newblock ``Bayesian nonparametric federated learning of neural networks,''
\newblock in {\em International conference on machine learning}. PMLR, 2019,
  pp. 7252--7261.

\bibitem{kingma2014adam}
Diederik~P Kingma and Jimmy Ba,
\newblock ``Adam: A method for stochastic optimization,''
\newblock {\em arXiv preprint arXiv:1412.6980}, 2014.

\end{thebibliography}

\end{document}